\documentclass{scrartcl}

\usepackage{amsmath, amsthm, amssymb}
\usepackage{mathrsfs}
\usepackage{url}
\usepackage{amssymb,amsfonts}
\usepackage{color}
\usepackage{shuffle}
\usepackage[usenames,dvipsnames]{xcolor}
\usepackage[colorlinks,backref]{hyperref}
\usepackage{graphicx}
\usepackage[font={small,it}]{caption}
\usepackage[flushleft]{threeparttable}

\usepackage[top=3cm, bottom=3cm, left=2.5cm, right=3cm]{geometry}

\newtheorem{theorem}{Theorem}
\theoremstyle{plain}

\newtheorem{lemma}[theorem]{Lemma}

\newtheorem{remark}[theorem]{Remark}

\theoremstyle{definition} 

\newtheorem{definition}[theorem]{Definition}

\newcommand{\R}{\mathbb{R}}
\newcommand{\C}{\mathbb{C}}

\newcommand{\spann}{\operatorname{span}}
\renewcommand{\Re}{\operatorname{Re}}
\renewcommand{\Im}{\operatorname{Im}}

\begin{document}

\title{Rotation invariants of two dimensional curves based on iterated integrals}
\author{Joscha Diehl, TU Berlin \footnote{diehl@math.tu-berlin.de}}
\maketitle

\begin{abstract}
  We introduce a novel class of rotation invariants
  of two dimensional curves based
  on \textit{iterated integrals}.
  The invariants we present are in some sense complete and we describe an algorithm to calculate them,
  giving explicit computations up to order six.
  We present an application to online (stroke-trajectory based) character recognition.
  This seems to be the first time in the literature that the use of iterated integrals of a curve is proposed for
  (invariant) feature extraction in machine learning applications.
\end{abstract}

\section{Introduction}

Data that can be represented as a curve in two dimensional Euclidean space
appears in many areas of pattern analysis.
The boundary (or silhouette) of an object in an image or video recording
can be stored as a (closed) two dimensional curve.
This has been used for example
for object recognition of everyday objects \cite{bibMokhtarian1995},
the mapping of (aerial) photographs to terrain maps \cite{bibMokhtarianMackworth1986,bibRodriguezAggarwal}
and to solving jigsaw puzzles \cite{bibFreemanGarder1964}.
In video recordings, the path of persons in the field of vision can be analyzed 
to detect anomalous behavior \cite{bibHuEtAl2004,bibZhouEtAl2008}
or predict future moves \cite{bibBennewitzEtAl2005}.
Hand gestures also naturally describe trajectories \cite{bibPsarrou2002}.
On a larger scale, the travel path, usually of a person or vehicel, in its environment
is easily recovered from time stamped location data, given for example either from GPS recordings or derived from mobile connectivity logs.
See the monograph \cite{bibZheng2011} for an overview of applications in this context.
In online (stroke-based) character recognition, the input stream
is usually the trajectory of pen-movements,
see for example the survey \cite{bibPlamondonSrihari2000}.

In some of these applications, there is no fixed coordinate system in which to orient the data,
so it becomes important to select features of the data that are invariant to rotation of the input.
In the example of online character recognition, thinking of multiuser tablets,
the angle from which the device is used does not have to be fixed.
This fact has to be taken into account when extracting features for the recognition task.

The construction of rotation invariants of \textit{images} has a long history.
Starting with the work of Hu \cite{bibHu} they are usually based on (centered) moments of the image.
Hu drew the connection to the classical problem of algebraic invariants \cite{bibCayley1854}
(see \cite{bibOlver1999} for a modern treatment) and was able to explicitly calculate seven invariants.
In subsequent work several different methods for deriving these invariants have been proposed, among them
Zernike moments \cite{bibWallinKubler1980},
the Fourier-Mellin transform \cite{bibLi1992} 
and Lie algebra methods \cite{bibDimai1999,bibSakataEtAl2004}.
The work of \cite{bibAbuMostafaPsaltis1984,bibFlusser2000} using complex moments inspired the method that we present.

Rotation invariant feature selection
of two-dimensional \textit{curves}
has also been treated in its own right.
Among the techniques are
Fourier series (of closed curves) \cite{bibGranlund1972,bibZahnRoskies1972,bibKuhlGiardina1982},
wavelets \cite{bibChuangKuo1996},
curvature based methods \cite{bibMokhtarianMackworth1986,bibCalabiEtAl1998}
and integral invariants \cite{bibManayCremers2006}.
Let us also mention the works
on codons \cite{bibHoffmanRichards1982}
the primal curvature sketch \cite{bibAsadaBrady1986}
and on Freeman chains \cite{bibFreeman1974},
which are usually not rotation invariant.

The invariants that we present are based on
the set of iterated integrals of a curve, which is usually denoted its \textit{signature}.
The signature as an object of study was first introduced by Chen \cite{bibChen1954}
and he showed that a curve is almost completely characterized by it \cite{bibChen1958}
(see \cite{bibHamblyLyons} for a recent generalization).
The importance of iterated integrals has by now become evident in areas such as control theory \cite{bibFliess1982},
ordinary differential equations and stochastic analysis \cite{bibLyonsIbero,bibLyonStFlour,bibFV,bibBaudoin2004}

The paper is structured as follows.
In Section \ref{sec:signature} we define iterated integrals and introduce (minimal) algebraic notations to deal with them.
Section \ref{sec:rotationinvariants} containts our main results, which give a means to calculate all rotation invariants based on the signature.
In Section \ref{sec:explicit} we carry out explicit computations for invariants up to order six, taking some care of algebraic independence.
Finally in Section \ref{sec:application}, as a prove of concept, we apply them to a simple character recognition problem.

Let us briefly mention some advantages of using iterated integrals for feature extraction. 
Curvature based methods \cite{bibMokhtarianMackworth1986,bibCalabiEtAl1998}
rely on the computation of the second derivative,
whereas the calculation of iterated integrals only needs the first derivative of the signal.
The latter is, of course, a more stable procedure.
In fact the stability of iterated integrals goes far beyond this fact, since they
can be compute even for signals that are \textit{nowhere differentiable}; think for example of the path of (a realization of) Brownian motion.
Moreover there exists a good approximation theory that for example usually provides convergence of piecewiese linear approximations (see \cite{bibFV}).
Although similar stability properties are also shared by Fourier methods \cite{bibGranlund1972,bibZahnRoskies1972,bibKuhlGiardina1982},
it is known (Chapter 1 in \cite{bibLyonsQian2002}) that the Fourier series representation is not well suited for highly oscillatory signals.
Considering a curve as an image in 2D (i.e. forgetting about the order in which it is drawn),
one can apply invariants used for \textit{images}.
Again, this will fail if the signal is highly oscillatory or overlaps frequently.
Moreover, it is for example almost impossible to distiguish the letter $M$ from a rotated letter $W$
\textit{considered as images}. The corresponding curves (as they are usually drawn) are completetely different though;
the first ``turn'' on the letter $M$ is to the right, the first one on the letter $W$ is to the left.

\newcommand{\TC}{T((\R^2))}
\newcommand{\TS}{T(\R^2)}
\newcommand{\TCc}{T((\C^2))}
\newcommand{\TSc}{T(\C^2)}

\section{The signature of a curve}
\label{sec:signature}

By a curve $X$ we will from now on denote
a continuous mapping $X: [0,T] \to \R^2$ of bounded variation.
Using geometric reasoning we can immediately recognize two rotation invariants of such a curve.
The first one is the Euclidean distance $E$
between startingpoint and endpoint.
The second one is the \textit{area} $A$ swept out by the closed curve,
which is obtained by connecting starting and endpoint via a straight line.

\begin{figure}[h]
  \includegraphics[width=\textwidth]{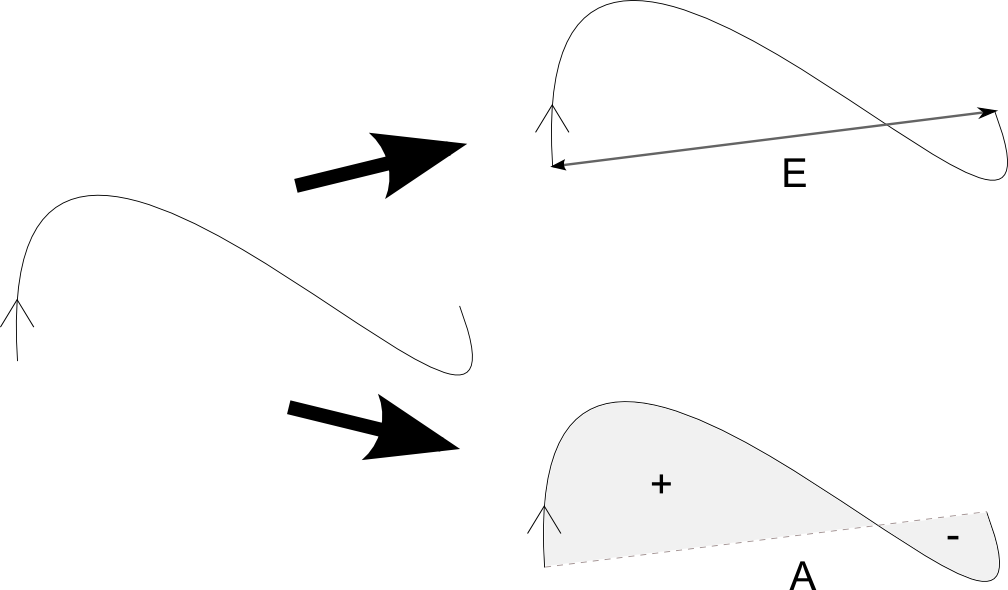}
  \caption{
    Two rotation invariants are shown.
    $E$ is the distance between starting and endpoint.
    $A$ is the area enclosed by the curve counted with orientation, as denoted by the plus and minus sign. }
\end{figure}

It turns out that both of these quantities can be written down
in terms of \textit{iterated integrals}
\footnote
{
  Since $X$ is of bounded variation
  the integrals are well-defined using classical Riemann-Stieltjes integration (see for example Chapter 6 in \cite{bibRudin1976}).
  This can be pushed much further though.
  In fact the following considerations are purely algebraic
  and hence hold for any curve for which a sensible integration theory exists.
  An example is two dimensional Brownian motion, which is almost surely 
  nowhere differentiable but nonetheless admits a
  Stratonovich integral.
}
of the curve,
\begin{align*}
  E^2 &= \frac{1}{2} \int_0^T \int_0^r dX^1_u dX^1_r + \frac{1}{2} \int_0^T \int_0^r dX^2_u dX^2_r \\
  A &= \frac{1}{2} \int_0^T \int_0^r dX^1_u dX^2_r - \frac{1}{2} \int_0^T \int_0^r dX^2_u dX^1_r.
\end{align*}
Indeed, the first equality is just an application of the integration by parts formula
(Theorem 6.22 in \cite{bibRudin1976}).
The second equality follows from Green's theorem (Theorem 10.33 in \cite{bibRudin1976}).
We thus hope to find other linear combinations of iterated integrals that also yield rotation invariants.

Let us introduce some algebraic notation in order to work with the collection of these integrals.
Denote by $\TC$, the space of formal power series in two \textit{non-commuting} variables $x_1, x_2$.
$\TC$ is the (algebraic) dual of $\TS$, the space of polyonomials in $x_1, x_2$,
where the pairing, denoted by $\langle \cdot, \cdot \rangle$ is defined by declaring all monomials to be orthonormal, e.g.
\begin{align*}
  \langle 1 + 4.3 x_1  + 7.9 x_1 x_2 - 0.2 x_1 x_2 x_1 + \cdots, x_1 x_2 \rangle = 7.9.
\end{align*}
We use the usual product of monomials, denoted by $\cdot$, extended to the whole space by bilinearity.
Note that $\cdot$ is also non-commuting.
See \cite{bibReutenauer1993} for background on these spaces.

We define the \textbf{signature} of $X$ to be
\begin{align*}
  S(X)_{0,T} := \sum x_{i_1} \dots x_{i_n} \int_0^T \int^{r_n} \dots \int_0^{r_2} dX^{i_1}_{r_1} \dots dX^{i_n}_{r_n},
\end{align*}
where the sum is taken over all $n \ge 0$ and all $i_1, \dots, i_n \in \{1,2\}$.
For $n=0$ the summand is, for algebraic reason, taken to be the constant $1$.
Note that $S(X)_{0,T}$ is an element of $\TC$.

It was proven in \cite{bibChen1958} (see \cite{bibHamblyLyons} for a generalization)
that the mapping $X \to S(X)_{0,T}$ is ``almost`` one-to-one.
In other words, a curve is completely characterized by its signature (modulo a ''tree-like`` path).

This fact should be compared to the fact
that the collection of all moments of a compactly supported density
completely determine that density (modulo modfications on sets of zero measure),
and gives us the justification to base the analysis of a curve in $\R^2$ entirely on its signature.

\section{Main results}
\label{sec:rotationinvariants}
The aim of this work is to find elements in $T(\R^2)$ that are invariant under rotations of $X$.
To be specific, for $\theta \in \R$, let 
\begin{align*}
  \bar X^\theta :=
  R(\theta) X :=
  \left(
  \begin{matrix}
    \cos( \theta ) & \sin( \theta ) \\
    -\sin( \theta ) & \cos( \theta )
  \end{matrix}
  \right) X,
\end{align*}
and compute its signature $S(\bar X^\theta)_{0,T}$.
We are then interested in $\phi \in T(\R^2)$ that satisfy for all $\theta \in \R$ and all curves $X$
\begin{align*}
  \langle S(X)_{0,T}, \phi \rangle = \langle S(\bar X^\theta)_{0,T}, \phi \rangle.
\end{align*}

\begin{definition}
  Denote such an element $\phi \in T(\R^2)$ as \textit{rotation invariant}.
\end{definition}

We are going to devise a method by which we will derive \textit{all} rotation invariants.
Our approach is inspired by \cite{bibFlusser2000}, where, in the setting of rotation invariant moments,
it is shown that is useful to work in the complex plane.
So from now on we allow elements in $\TC$ and $\TS$ to have complex coefficients,
i.e we work with $\TCc$ and $\TSc$.
The proofs of this section can be found in the appendix.
\begin{theorem}
  \label{thm:invariants}

  Let $n\ge 2$ and $i_1, \dots, i_{n} \in \{1,2\}$ be such that
  \begin{align}
    \label{eq:equalPlusMinus}
    \#\{ k : i_k = 1 \} = \#\{ k : i_k = 2 \}.
  \end{align}

  Then $\phi := c_{i_1 \dots i_n }$ is rotation invariant, where
  \begin{align*}
    c_{i_1 \dots i_n } &:= z_{i_1} \cdot z_{i_2} \cdot \ldots \cdot z_{i_n} \\
    z_1 &:= x_1 + i x_2 \\
    z_2 &:= x_1 - i x_2.
  \end{align*}
\end{theorem}

The just described method suffices to find \textit{all} rotation invariants:
\begin{theorem}
  \label{thm:completeness}
  Let $\phi \in T(\R^2)$ be rotation invariant.
  Then we can write $\phi$ as the finite sum of
  the invariants given in Theorem \ref{thm:invariants}.
\end{theorem}

\section{Explicit construction up to order six}
\label{sec:explicit}

Before giving explicit expressions for invariants up to order six we have to deal with
one particularity of the signature.
Namely, the fact that every polynomial in elements of the signature is
actually a linear function of (different) elements of the signature.

\subsection{Shuffle identity}
The shuffle product $\shuffle$ on $\TSc$ is commutative and extended by bilinearity from its definition on monomials.
The shuffle product of two monomials
consists in the sum of all possible ways of interleaving the two monomials
while keeping their respective order.
For example
\begin{align*}
  x_1 \shuffle x_2 &= x_1 x_2 + x_2 x_1 \\
  x_1 x_2 \shuffle x_1 x_2
  &=  
  x_1 x_2 x_1 x_2
  +
  x_1 x_1 x_2 x_2
  +
  x_1 x_1 x_2 x_2
  +
  x_1 x_1 x_2 x_2
  +
  x_1 x_1 x_2 x_2
  +
  x_1 x_2 x_1 x_2 \\
  &=
  2 x_1 x_2 x_1 x_2
  +
  4 x_1 x_1 x_2 x_2.
\end{align*}
See \cite{bibReutenauer1993} for a completely rigorous definition.
The significance of this product stems from the following \textit{shuffle identity}, which is proven for example in \cite{bibRee1958}.
\begin{lemma}
  \label{lem:shuffle}
  Let $X: [0,T] \to \R^2$ be a path of bounded variation.
  Then for every $a,b \in \TSc$ we have
  \begin{align*}
    \langle S(X)_{0,T}, a \rangle 
    \langle S(X)_{0,T}, b \rangle
    =
    \langle S(X)_{0,T}, a \shuffle b \rangle 
  \end{align*}
\end{lemma}
\begin{remark}
  The shuffle identity represents, on an algebraic level, the fact
  that integrals obey the integration by parts rule.
  For example, taking $a = x_1, b = x_2$, we get
  \begin{align*}
    \left( \int_0^T dX^1_r \right) \cdot \left( \int_0^T dX^2_r \right)
    =
    \int_0^T \int_0^r dX^1_u dX^2_r
    +
    \int_0^T \int_0^r dX^2_u dX^1_r.
  \end{align*}
\end{remark}

\subsection{Computation to order six}
From Theorem \ref{thm:invariants} it is clear that rotation invariants only exist on levels of even order.
On the level of order two, the theorem gives the (complex) invariants
\begin{align*}
  c_{12} &= x_1 x_1 - i x_1 x_2 + i x_2 x_1 + x_2 x_2 \\ 
  c_{21} &= x_1 x_1 + i x_1 x_2 - i x_2 x_1 + x_2 x_2.
\end{align*}
We are interested in real invariants, so we take real and imaginary part
\begin{align*}
  \Re c_{12} &= x_1 x_1 + x_2 x_2 \\ 
  \Im c_{12} &= - x_1 x_2 + x_2 x_1\\ 
  \Re c_{21} &= x_1 x_1 + x_2 x_2\\
  \Im c_{21} &= x_1 x_2 - x_2 x_1.
\end{align*}
A linear basis for them is
\begin{align*}
  I_1 &:= x_1 x_1 + x_2 x_2 \\
  I_2 &:= x_1 x_2 - x_2 x_1,
\end{align*}
which we recognize (modulo a prefactor of $1/2$) as the geometric invariants we saw at the beginning of Section \ref{sec:signature}.
It is unclear whether there exist geometrical interpretations for higher order invariants.

Proceeding to order four, by the shuffle identity (Lemma \ref{lem:shuffle} below),
the information given by $I_1 \shuffle I_1, I_2 \shuffle I_2$ and $I_1 \shuffle I_2$ is already known
from the invariants of order two.
So, we start by taking a basis for $\spann_{\R}\{ I_1 \shuffle I_1, I_2 \shuffle I_2, I_1 \shuffle I_2 \}$.
We then extend it to a basis
for $\spann_{\R}\{ I_1 \shuffle I_1, I_2 \shuffle I_2, I_1 \shuffle I_2, \Re \phi_i, \Im \phi_i \}$,
where the $\phi_i$ range over the invariants of order four given by Theorem \ref{thm:invariants}
(i.e. $c_{1122}, c_{1212}, c_{1221}, c_{2111}, c_{2121}$ and $c_{2211}$).
The additional basis vectors are given by
\begin{align*}
  I_3 &:= 1122 - 1212 - 1221 - 2121 \\
  I_4 &:= - 1212 + 2112 + 2211 \\
  I_5 &:= 1212 - 2112 - 2211.
\end{align*}
where, to shorten notation, we denote $i_1 \dots i_n := x_{i_1} \cdot \ldots \cdot x_{i_n}$.

Repeating the same procedure for order six we get
\begin{align*}
  I_6 &:= 111222 - 112122 - 112212 - 112221 - 121212 - 121221 - 122121 \\
    &\qquad - 211122 - 211212 - 211221 - 212121 - 221112 - 221121 - 222111 \\
  I_7 &:=  - 112122 - 112221 - 121122 - 121221 - 122112 - 122211 - 211122 \\
    &\qquad - 211221 - 212112 - 212211 - 221112 - 221211 - 222111 \\
  I_8 &:= 111222 - 112122 - 112212 - 112221 - 121212 - 121221 - 122121 \\
    &\qquad - 211122 - 211212 - 211221 - 212121 - 221112 - 221121 - 222111 \\
  I_9 &:=  - 112122 + 112221 - 121122 + 121221 - 122112 + 122211 - 211122 \\
    &\qquad + 211221 - 212112 + 212211 - 221112 + 221211 - 222111 \\
  I_{10} &:= 111222 + 112122 + 112212 + 112221 + 121212 + 121221 + 122121 \\
    &\qquad + 211122 + 211212 + 211221 + 212121 + 221112 + 221121 + 222111 \\
  I_{11} &:=  - 112122 - 112221 + 121122 - 121221 + 122112 - 122211 + 211122 \\
    &\qquad - 211221 + 212112 - 212211 + 221112 - 221211 + 222111 \\
  I_{12} &:= 111222 + 112122 + 112212 + 112221 + 121212 + 121221 + 122121 \\
    &\qquad + 211122 + 211212 + 211221 + 212121 + 221112 + 221121 + 222111 \\
  I_{13} &:= 112122 - 112221 - 121122 - 121221 - 122112 - 122211 - 211122 \\
    &\qquad - 211221 - 212112 - 212211 - 221112 - 221211 - 222111 \\
  I_{14} &:= 111222 + 112122 + 112212 + 112221 + 121212 + 121221 + 122121 \\
    &\qquad + 211122 + 211212 + 211221 + 212121 + 221112 + 221121 + 222111 \\
  I_{15} &:= 112122 - 112221 - 121122 - 121221 - 122112 - 122211 - 211122 \\
    &\qquad - 211221 - 212112 - 212211 - 221112 - 221211 - 222111.
\end{align*}

\section{Application to character recognition}
\label{sec:application}

We present the application to a simple classification problem
on the dataset \textit{pendigits} \cite{bibAlimoglu1996}, which consists of handwritten
digits by $44$ writers, $30$ of which are used for training ($7493$ samples) and $14$ of which are used for testing ($3497$ samples).
Input was recorded on a tablet device and hence \textit{stroke} data given.
Since some inputs consist of multiple strokes (the pen is lifted from the device and then touches again at a different location),
we have to convert the signal into a continuous curve first.
We achieve this by just connecting endpoint of the previous and startingpoint of the next stroke by a straight line.
To verify that our method is indeed rotation invariant, we rotated each digit in the test-set by a random angle.

We used support vector machines with linear and rbf kernel, where for the latter
we optimized the parameters on the training set via $3$-fold crossvalidation.
The results are presented in Table \ref{tab:results}.

\begin{table}
  \center
  \begin{threeparttable}
    \caption{Pendigits classification error}
    \label{tab:results}
    \begin{tabular}{|l||r|r|r|r|r|}
    \hline
    ~      & Order 2 (2) & Order 4 (5) & Order 4 (8) & Order 6 (15) & Order 6 (28) \\ \hline
    linear & 38.34                & 15.81                & 15,58                & 9.98                  & 7.83                  \\ \hline
    rbf    & 37.60                & 9.27                 & 8.18                 & 4.91                  & 4.60                  \\ \hline
    \end{tabular}\par
    \begin{tablenotes}
      \scriptsize
      \item 
  Classification error using features of different order.
  The numbers in brackets denote the number of features.
  For order $4$ and $6$ two different sets of features were used.
  The lower number was obtained by making the features linearily and algebraically independent as described in Section \ref{sec:explicit}.
  The higher number was obtained using all the invariants given by Theorem \ref{thm:invariants}.
    \end{tablenotes}
  \end{threeparttable}
\end{table}

\section{Appendix}
\label{sec:appendix}
The proofs of Theorem \ref{thm:invariants} and Theorem \ref{thm:completeness} follow.
We shall need, for $n\ge 0$, the projection operator $\pi_n$ which
sets all coefficients of a polynomial or formal series to zero,
except the ones belonging to monomials of order $n$, which it leaves unchanged; for example
\begin{align*}
  \pi_2( 1 + x_1 x_2 + x_2 x_1 + x_1^3 x_2 ) = x_1 x_2 + x_2 x_1.
\end{align*}

\begin{proof}[Proof of Theorem \ref{thm:invariants}]
  Let $Z = Z(X)$ be defined as
  \begin{align*}
    Z^1_r &:= X^1_r + i X^2_r \\
    Z^2_r &:= X^1_r - i X^2_r.
  \end{align*}

  Let $\bar Z^\theta = Z(\bar X^\theta)$.
  Then $\bar Z^{\theta, 1} = e^{-i\theta} Z^1$, $\bar Z^{\theta, 2} = e^{i \theta} Z^2$,
  and hence for $i_1, \dots, i_n \in \{1,2\}$ we have
  \begin{align*}
    \langle S( \bar Z^\theta )_{0,T}, x_{i_1} \cdot \ldots \cdot x_{i_n} \rangle
    &=
    \int_0^T \int_0^{r_n} \dots \int_0^{r_2} d\bar Z^{\theta, i_1}_{r_1} \dots d\bar Z^{\theta,i_n}_{r_n} \\
    &= 
    \int_0^T \int_0^{r_n} \dots \int_0^{r_2} e^{ (-1)^{i_1} i \theta} dZ^{i_1}_{r_1} \dots e^{ (-1)^{i_n} i \theta} dZ^{i_n}_{r_n} \\
    &=
    e^{ i \theta \left( \#\{ k : i_k = 2 \} - \#\{ k : i_k = 1 \} \right) }
    \langle S( Z )_{0,T}, x_{i_1} \cdot \ldots \cdot x_{i_n} \rangle.
  \end{align*}

  Now, on the other hand,
  \begin{align*}
    \langle S( Z ) )_{0,T}, x_{i_1} \cdot \dots x_{i_n} \rangle
    &=
    \langle S( X ) )_{0,T}, c_{i_1 \dots i_n} \rangle \\
    \langle S( \bar Z^\theta ) )_{0,T}, x_{i_1} \cdot \dots x_{i_n} \rangle
    &=
    \langle S( \bar X ) )_{0,T}, c_{i_1 \dots i_n} \rangle,
  \end{align*}
  which shows that $c_{i_1 \dots i_n}$ is rotation invariant if it satisfies \eqref{eq:equalPlusMinus}.
\end{proof}

\begin{proof}[Proof of Theorem \ref{thm:completeness}]
  Let $n\ge 1$ and let $\phi_n := \pi_n \phi$ be the projection on the $n$-th level.
  Since the invariants given in Theorem \ref{thm:invariants} are homogeneous, it is enought
  to show the statement for $\phi_n$.

  By Lemma \ref{lem:projectionIsInvariant} we have that $\phi_n$ is also rotation invariant
  and by Lemma \ref{lem:span} we can write $\phi_n$ as
  \begin{align*}
    \phi_n
    =
    \sum_{i_1, \dots, i_n \in \{1,2\}} a_{i_1, \dots, i_n} c_{i_1, \dots, i_n},
  \end{align*}
  for some uniquely determined $a_{i_1, \dots, i_n} \in \C$.
  Then
  \begin{align*}
    \langle S(\bar X^\theta)_{0,T}, \phi_n \rangle
    &=
    \sum_{i_1, \dots, i_n \in \{1,2\}} a_{i_1, \dots, i_n} \langle S(\bar X^\theta)_{0,T}, c_{i_1, \dots, i_n} \rangle \\
    &=
    \sum_{i_1, \dots, i_n \in \{1,2\}} a_{i_1, \dots, i_n}
    e^{ i \theta \left( \#\{ k : i_k = 2 \} - \#\{ k : i_k = 1 \} \right) }
    \langle S(X)_{0,T}, c_{i_1, \dots, i_n} \rangle.
  \end{align*}
  On the other hand, since $\phi_n$ is rotation invariant, we have
  \begin{align*}
    \langle S(\bar X^\theta)_{0,T}, \phi_n \rangle
    =
    \langle S(X)_{0,T}, \phi_n \rangle
    =
    \sum_{i_1, \dots, i_n \in \{1,2\}} a_{i_1, \dots, i_n} \langle S(X)_{0,T}, c_{i_1, \dots, i_n} \rangle.
  \end{align*}
  Combining, we arrive at
  \begin{align}
    \label{eq:zero}
    \sum_{i_1, \dots, i_n \in \{1,2\}}
    \left(
      a_{i_1, \dots, i_n}
      -
      e^{  i \theta \left( \#\{ k : i_k = 2 \} - \#\{ k : i_k = 1 \} \right) }
      a_{i_1, \dots, i_n} \right) \langle S(X)_{0,T}, c_{i_1, \dots, i_n} \rangle
    =
    0,
  \end{align}
  for all curves $X$ of bounded variation.

  We need to show: $a_{i_1, \dots, i_n} = 0\ $ if
  $\ \#\{ k : i_k = 2 \} - \#\{ k : i_k = 1 \} \not= 0$.
  By \eqref{eq:zero} this follows if we can show the existence of paths $X^{(1)}, \dots, X^{(2^{n})}$ such that the vectors, defined as
  \begin{align*}
    v_k :=
    ( \langle S(X^{(k)}_{0,T}, c_{1 1 \dots1 1 2} \rangle,  \langle S(X^{(k)}_{0,T}, c_{1 1 \dots 1 2 1} \rangle,  \dots, \langle S(X^{(k)}_{0,T}, c_{2 2 \dots 2 2 2} \rangle ),\quad k=1,\dots,2^n,
  \end{align*}
  are linearly independent. But this follows from Lemma \ref{lem:GspansT} and Lemma \ref{lem:span}.
\end{proof}

\begin{lemma}
  \label{lem:projectionIsInvariant}
  Let $\phi \in T(\R^2)$ be rotation invariant.
  Then $\phi_n := \pi_n( \phi )$ is rotation invariant for all $n\ge1$.
\end{lemma}
\begin{proof}
  Let $N$ be the order of $\phi$.
  Let $X$ be some curve and
  for let the dilation by $\alpha \in \R$ be given as $X_\alpha := \alpha X$.
  Then
  \begin{align*}
    \sum_{n=1}^N \alpha^i \langle S(X)_{0,T}, \pi_n \phi \rangle
    &=
    \sum_{n=1}^N \langle S(X_\alpha)_{0,T}, \pi_n \phi \rangle \\
    &=
    \langle S(X_\alpha)_{0,T}, \phi \rangle \\
    &=
    \langle S(\bar X^\theta_\alpha)_{0,T}, \phi \rangle \\
    &=
    \sum_{n=1}^N \langle S(\bar X^\theta_\alpha)_{0,T}, \pi_n \phi \rangle \\
    &=
    \sum_{n=1}^N \alpha^i \langle S(\bar X^\theta)_{0,T}, \pi_n \phi \rangle.
  \end{align*}
  Since this holds for all $\alpha \in \R$ we have for all $n\ge 1$ and all curves $X$
  \begin{align*}
    \langle S(X)_{0,T}, \pi_n \phi \rangle
    =
    \langle S(\bar X^\theta)_{0,T}, \pi_n \phi \rangle.
  \end{align*}
  Hence $\pi_n \phi$ is rotation invariant for all $n\ge 1$.
\end{proof}

\begin{lemma}
  \label{lem:span}
  For every $n\ge1$
  \begin{align*}
    \{ c_{i_1, \dots, i_n} : i_1, \dots, i_n \in \{1,2\} \}
  \end{align*}
  is a basis for $\pi_n T( \R^2 )$.
\end{lemma}
\begin{proof}
  Denote by $v_{i_1, \dots, i_n}$ the vector in $\C^n$ corresponding to the coefficients
  of the monomials of $c_{i_1, \dots, i_n}$, $i_1, \dots, i_n \in \{1,2\}$, where we order the monomials lexicographically.
  Let $M^n$ be the $n\times n$ matrix constructed from the $v_{i_1, \dots, i_n}$ where
  we order the rows according to the lexicographically order of $i_1 \dots i_n$.
  We have to show that $M^n$ is of full rank.

  This is obviously true for $n=1$.
  Let be true for an arbitrary $n$.
  Then
  \begin{align*}
    M^{n+1}
    =
    \left(
      \begin{matrix}
        M^n & i M^n \\
        M^n & -i M^n
      \end{matrix}
    \right),
  \end{align*}
  from which we see that $M^{n+1}$ has full rank.
\end{proof}

\begin{lemma}
  \label{lem:GspansT}
  Let $n\ge 1$.
  Then
  \begin{align}
    \label{eq:GspansT}
    \spann_{\C}\{ \pi_n( S(X) _{0,1} ) : X \text{ continuous and of bounded variation } \}
    = \pi_n \TCc.
  \end{align}
\end{lemma}
\begin{proof}
  It is clear by definition that the left hand side of \eqref{eq:GspansT} is included in $\pi_n \TCc$.
  We show the other direction and use ideas of Proposition 4 in \cite{bibCassFriz}. 
  Let $x_{i_n} \cdot \ldots \cdot x_{i_1} \in \pi_n \TCc$ be given.
  Let $X$ be the piecewise linear path, that results from the concatenation of the vectors $t_1 e_{i_1}, t_2 e_{i_2}$ up to $t_n e_{i_n}$,
  where $e_i, i=1,2$ is the standard basis of $\R^2$.
  Its signature is given by (see for example Chapter 6 in \cite{bibFV})
  \begin{align*}
    S(X)_{0,1} = \exp( {t_n x_{i_n}} ) \cdot \ldots \cdot \exp( t_1 v_{i_1} ) =: \phi(t_1, \dots, t_n),
  \end{align*}
  where the exponential function is defined by its power series.
  Then
  \begin{align*}
    \label{eq:derivative}
    \frac{d}{dt_n} \dots \frac{d}{dt_1} \phi(0,\dots,0) = x_{i_n} \cdot \ldots \cdot x_{i_1}.
  \end{align*}
  Combining this with the fact that left hand side of \eqref{eq:GspansT} is a closed set we get that
  \begin{align*}
    x_{i_n} \cdot \ldots \cdot v_{i_1}
    \in 
    \spann_{\C}\{ \pi_n( S(X) _{0,1} ) : X \text{ continuous and of bounded variation } \}.
  \end{align*}

  These elements span $\pi_n \TCc$, which finishes the proof.
\end{proof}

\end{document}